\documentclass{article}
\usepackage[preprint]{neurips_2026}

\usepackage[utf8]{inputenc} 
\usepackage[T1]{fontenc}    
\usepackage{hyperref}       
\usepackage{url}            
\usepackage{booktabs}       
\usepackage{amsfonts,amsmath,amsthm}       
\usepackage{nicefrac}       
\usepackage{microtype}      
\usepackage{xcolor}         
\usepackage{graphicx,tikz}
\usepackage{float}

\newtheorem{theorem}{Theorem}[section]
\newtheorem{lemma}[theorem]{Lemma}
\newtheorem{remark}[theorem]{Remark}

\title{Feature Selective Model Collapse in Diffusion Models: Total Replacement versus Fixed-Budget Training}

\author{%
  Hanna~Malet\\
CEREMADE CNRS\\
  Université Paris Dauphine - PSL\\
  Paris 75116 FRANCE\\
 \& \\
Cornell University\\
  Ithaca, NY, USA\\
\And
  Gabriel~Turinici \\
CEREMADE CNRS\\
  Université Paris Dauphine - PSL\\
  Paris 75116 FRANCE\\
  \texttt{gabriel.turinici@dauphine.fr} \\ \url{https://turinici.com} \\
}

\begin{document}

\maketitle

\begin{abstract}
Model collapse arises when generative models are trained on synthetic data produced by earlier models. The phenomenon has attracted considerable attention because of its societal and technical implications. However, previous studies have reached seemingly contradictory conclusions: replacing real data with synthetic data causes collapse (Shumailov et al.), yet accumulating real data alongside synthetic data can prevent it. For  diffusion models, 
we study an intermediate regime typical of finite-budget pipelines: all past datasets and the real data are kept, but each new model is trained on a fixed-size sample from this growing pool, so the real fraction vanishes without any data being removed. 
Experiments on a 2D spiral dataset as well as the image benchmarks (MNIST, Fashion-MNIST, and CIFAR-10) show that replacement protocol degrades dataset rapidly as in the literature, whereas the fixed budget degrades only partially, sparing some features. A linear-response model of the multi-generational parameter dynamics, analyzed by stochastic recursion, confirms that the two protocols differ: some features will be fragile and lost within a few generations for both protocols, while some will be robust and  preserved over practically unbounded horizons under the fixed budget protocol.
\end{abstract}

\section{Introduction}
\label{sec:introduction}

Generative models are increasingly trained on data that already contains the output of earlier
 models. As synthetic content becomes ubiquitous and training pipelines rely on ever
larger datasets collected over long periods of time, understanding the long-term consequences of
such recursive training has become both a practical and a theoretical necessity with societal and technical implications. 
\citet{shumailov2024curserecursiontraininggenerated} (see also \citet{dohmatob2024tale}) showed that this feedback loop degrades
successive models until the tails of the original distribution disappear, a phenomenon they
called \textit{model collapse}. Their main protocol is one of recursive replacement: each generation
is trained on the output of the previous one, and the real dataset is gone after the first step.

Subsequent work reached conclusions that appear to contradict this one.
\citet{gerstgrasser2024modelcollapseinevitablebreaking} argued that if synthetic data
accumulate alongside the real data instead of replacing it collapse no longer occurs. \citet{seddik2024badtrainingsyntheticdata} made collapse depend on
the synthetic-to-real ratio rather than on the recursion itself, and
\citet{barzilai2026modelsdontcollapseconsistency} showed that for maximum likelihood
estimation collapse can be avoided even when the fraction of real data vanishes, while also
exhibiting cases where it cannot. \citet{pmlr25_model_collapse} approach the question from the
mitigation side, editing human-produced text rather than regenerating it. Closer to us in model
class, \citet{yoon2025chaindiffusion} and \citet{bohacek2025nepotistically} document collapse in
finetuned Stable Diffusion over short chains. 

We investigate a recursive-training protocol that is intermediate between those of
\citet{shumailov2024curserecursiontraininggenerated} and
\citet{gerstgrasser2024modelcollapseinevitablebreaking}. As in the latter, synthetic data are
accumulated rather than replacing the original dataset. However, unlike that setting, here the
training-set size is kept fixed: at each generation, a constant number of examples is
resampled from the growing data pool. This regime, which naturally arises under a finite training budget, was introduced as
\emph{accumulate-subsample} by \citet{kazdan2025collapse} and examined empirically for diffusion models by \citet{shi2025closer}; we complement these works with a per-feature theoretical analysis.

Our contributions are both empirical and theoretical. On the empirical side, we focus on the class
of diffusion models that has received comparatively less
attention in the model-collapse literature. We recursively train diffusion models over several generations, first on a toy 2D spiral distribution and then on MNIST,
Fashion-MNIST \& CIFAR-10.
 Under total replacement, we recover the rapid deterioration reported in previous
work. Under accumulation with a fixed dataset size and constant training budget, the deterioration is only partial and
affects specific features, leading to a behavior that differs both from complete collapse and
from the bounded-error regime predicted under unbounded accumulation.

On the theoretical side, we introduce a system of equations describing the multi-generational parameter dynamics of recursively trained diffusion models. Like \citet{bertrand2024stability}, we linearize the retraining map in parameter space; unlike their stability analysis, which requires a large enough fraction of real data, our model covers a real fraction vanishing as $1/(n+1)$ and a random re-initialization at each generation.  Using stochastic analysis, we establish rigorous convergence results that explain both the rapid degradation observed under total replacement and, in the fixed-budget accumulation regime, the heterogeneous evolution of reconstruction quality across different features.
We also derive per-feature halving times which shows that recursive training first eliminates 'fragile' features providing thus quantitative guidance on when real-data refreshes and related mitigation measures become necessary.

A detailed positioning relative to prior work is 
presented in Appendix~\ref{sec:related-work}. In short, the works differ less about the recursion itself than
about how the training set is built at each step; the fixed training-set size removes the
vanishing-noise-share weighting on which the protective effect of accumulation relies.

\section{Methodology}
\label{sec:methodology}

We train diffusion models with DDPM training \citep{ho2020denoising} and DDIM sampling
\citep{song2021denoising}. See appendix for architecture details. Each model is trained on a
dataset and results in a new model called a "generation". Let $\mathcal{D}$ be the real dataset (also denoted $\mathcal{D}_0$), $d$  the number of model parameters
and let $S_n \in \mathbb{R}^d$ denote the generative model parameters at generation $n$; once $S_n$ is obtained, 
a
synthetic dataset of size $n_d$, denoted
$\mathcal{D}_n$, is sampled using the $S_n$ as model parameters.

Two data protocols are used: the {\bf last generation} protocol where generation $n+1$ is
trained on $\mathcal{D}_n$ alone and the {\bf fixed budget} protocol where generation $n+1$ is
trained on a sample of fixed size $n_d$ drawn from the pool
$\mathcal{D}\cup \mathcal{D}_1\cup\cdots\cup \mathcal{D}_n$; the pool grows but the training set
does not, so the probability that a training example is real decreases with $n$ although
no data is ever discarded. This protocol can describe for instance the data accumulation on internet, each new model trains only on a part of available corpora.

All models start from a set of parameters $Z_n$ sampled at random from some distribution
$\nu$.

As {\bf datasets} we use a toy 2D spiral dataset and then MNIST, Fashion-MNIST \& CIFAR-10.

\noindent {\bf Parameters} 
We use $n_d \sim 10^4$ samples, $n_{epochs} \sim 100-1000$ epochs and $n_{gen} \sim 20-100$ generations, cf. Appedix for details. To
sample from the model we use a Exponential Model average with parameter $0.8$ that smooths model
parameters.

\noindent {\bf Metrics.} To evaluate the model collapse for image datasets we use three metrics: Fréchet Inception
Distance (FID), pixel-FID (same formula as FID but without using Inception features) and
Wasserstein-2 metric, denoted $\mathcal{W}_2$ at pixel level. For the simple 2D spiral 
dataset 
neither the Inception features nor the pixel-space Fréchet distance make sense so we use 
$\mathcal{W}_2$ computed exactly by solving the discrete optimal transport problem; in this case the Fréchet distance is computed between
the Gaussian approximations of the two clouds in $\mathbb{R}^2$.

\section{Modeling the multi-generational dynamics and Diffusion} \label{sec:appendix_model}

Diffusion models are procedures that define two mappings:  

- {\bf thse 'training' phase} $\mathcal{T}$ that 
has as input a couple (initial value $Z$, distribution $\mathcal{D}$)  and as output a parameter denoted here
$S = \mathcal{T}_{Z,\mathcal{D}}\in \mathbb{R}^d$

- and an inverse mapping called {\bf 'sampling'} $\mathcal{S}$
that associates to a parameter $S\in\mathbb{R}^d$ a distribution denoted $\mathcal{S}_S$.

Both phases involves a neural network whose parameters are $S\in\mathbb{R}^d$. 
Note that in the following we will not make a distinction between the dataset $\mathcal{D}$ and the distribution it generates. Also note that 
we fixed the architecture of the neural network and the a priori distribution from which the sampling is performed to obtain $\mathcal{S}_S$. We will suppose in the following  that
\begin{align}
&\textit{Diffusion model hypothesis :} 
\text{ for any dataset of interest }
    \mathcal{D} \text{ there exists an unique (but unknown)}
\nonumber \\   & \text{parameter }
    S \in \mathbb{R}^d
    \text{ such that } \mathcal{D} = \mathcal{S}_S.
\label{hyp:dataset_representation}
\end{align}
Of course, the hypothesis \ref{hyp:dataset_representation} is too strong to be true in general but if the neural network capacity is large enough one can train a diffusion model to represent the dataset 
$\mathcal{D}$ with enough accuracy so we expect the error in \eqref{hyp:dataset_representation} to be negligible.

\subsection{The simplest Diffusion training equations}

To model the diffusion training we need to know how the neural network parameter $S$ depends, through the training, on the training (target) dataset 
$\mathcal{D}$ and on the initialization $Z$. 
Note that because of the hypothesis \eqref{hyp:dataset_representation} we can restrict ourselves to mappings involving only elements in $\mathbb{R}^d$ by replacing $\mathcal{D}$ by its representative $T$ such that $\mathcal{S}_T= \mathcal{D}$. 

Our goal is not to describe the full nonlinear dependence of $S_{train} = \mathcal{T}_{Z,\mathcal{D}=\mathcal{S}_T}$ on $(Z,T)$, but rather
to derive a crude first-order approximation in the spirit of Neural Tangent Kernel (NTK) \cite{jacot2018NTK}
analyses. In particular, we seek a model that captures the leading-order effect of the
initialization while aggregating all higher-order phenomena into a residual term.

Observe first that when the initialization $Z$ coincides with the target parameter $T$,
the procedure is already at equilibrium and therefore if 
$Z=T$ then $S_{train} \sim T$ with possible small oscillations due to noise terms. Hence, the map
\begin{align}
\Phi : Z \mapsto S_{train}
\end{align}
admits an almost 'fixed point' at $T$ in the sense that 
$\Phi(T) = T+\varepsilon$ where 
$\varepsilon$ is some noise taken here with zero average and independent of $T$.
 Assuming that $\Phi$ is sufficiently regular,
a first-order Taylor expansion around $y$ yields
\begin{align}
S_{train}-T
=
\Phi(Z)-\Phi(T)
\approx
J_T (Z-T),
\end{align}
where $J_T$ denotes the Jacobian of $\Phi$ evaluated at $T$.

The diffusion dynamics are stochastic and only approximately captured by their
linearization. Moreover, higher-order nonlinear effects are neglected in the above
expansion. We therefore introduce an additive fluctuation term and write
\begin{equation}
S_{train}-T = M(Z-T) + \varepsilon ,
\ \varepsilon \sim \mu, \ \mathbb{E}[\varepsilon]=0,
\ Var[\varepsilon]< \infty,
\label{eq:first_order_model}
\end{equation}
where $\varepsilon$ is a centered random variable accounting for both the intrinsic
diffusion noise and the approximation error induced by linearization. We assume it follows some centered distribution $\mu$ with finite second order moment.

The matrix $M$ should be interpreted as a local sensitivity operator describing how
a perturbation of the initialization propagates to the final parameter. Under a local
quadratic approximation of the discrepancy functional around $T$, the dominant
second-order information is encoded by a Hessian-like operator. This motivates
restricting $M$ to be symmetric positive definite. 
Such a choice is consistent with locally stable gradient-driven dynamics and excludes
purely rotational components that are not expected to appear at first order. Moreover, since the training is supposed to improve the initial quality we can assume that $M$ is somehow a contraction which means that its spectrum is in $]0,1[$. 
Consequently, among simple first-order models, the linear-response ansatz:
\begin{align}
S_{train}-T = M(Z-T) + \varepsilon ,
\  M=M^T, \  spectrum(M) \in ]0,1[,
\ \varepsilon \sim \mu, \ \mathbb{E}[\varepsilon]=0,
\ Var[\varepsilon]< \infty,
\end{align}
appears as a natural approximation. The deterministic component captures the
leading-order dependence on the initialization, while the noise term absorbs
higher-order nonlinear effects and the stochasticity intrinsic to the diffusion
procedure.

\subsection{The simplest model for dataset evolution:  Linearity with Respect to Data Mixtures}

The last ingredient of out model is a way to describe how the dataset evolves from one generation to the other and more precisely how its representative in $\mathbb{R}^d$ evolves. 

For the ''last generation'' protocol (cf. Section~\ref{sec:methodology}), the dependence is easy to set: if the previous generation produces parameter $S$ then the dataset will correspond to $S$ (at first order). For the ''fixed budget'' protocol we need to explain how the new dataset depends on the old one and on the newly generated data. 

We next consider a target distribution obtained as a mixture of two datasets i.e. a dataset taking $\lambda$ percentage of data from dataset $\mathcal{S}_u$ and the rest from $\mathcal{S}_v$. As distributions this mixture is formally denoted: 
$
\lambda \mathcal{S}_u + (1-\lambda)\mathcal{S}_v,
\ \lambda\in[0,1]$.
Within the same first-order approximation framework, we assume that the map
$S \mapsto \mathcal{S}_S$ is locally linear around the parameters of interest.
Under this assumption, convex combinations of distributions are approximately
represented by convex combinations in parameter space. 
Thus, in the same spirit as NTK-type linearizations, convex interpolation in data space
induces, to first order, a convex interpolation in parameter space.
and consequently, the parameter associated with the mixed dataset is approximated by
$\lambda u +(1-\lambda)v$. More precisely, up to some negligible error, we will assume:
\begin{align}
\lambda \mathcal{S}_u + (1-\lambda)\mathcal{S}_v
=
\mathcal{S}_{\lambda u +(1-\lambda)v}.
\end{align}

\section{Theoretical results : convergence and large perturbation analysis} \label{appendix:theory}

We analyze here the convergence dynamics of the multi-generation diffusion models under the two data protocols. Denote $S_n \in \mathbb{R}^d$ the parameters of the  diffusion neural networks at generation $n$, $T_n \in \mathbb{R}^d$ the parameters of the target distribution at generation $n$, $Z_n$ the initialization of the generation $n+1$ and $\varepsilon_n$ the noise at generation $n+1$ training. Under the
assumptions in Section~\ref{sec:appendix_model} we obtain the following recursion:
\begin{align}
&\text{Training equation (both protocols) : } S_{n+1}-T_n = M(Z_n-T_n) + \varepsilon_n
\label{eq:training_equation}
\\
& \varepsilon_n \sim \mu \text{ i.i.d., } \ \ \forall n :  
\mathbb{E}[\varepsilon_n]=0, \ Var[\varepsilon_n] < \infty
\nonumber \\
& \text{ Fixed budget protocol : }
(n+1)T_{n+1} = T_n n + S_{n+1}
\label{eq:fixed_budget}
\\
&\text{ Last generation protocol : }
T_{n+1} = S_{n+1}.
\label{eq:last_generation}
\end{align}
Note that after a change of basis we can assume matrix $M$ diagonal, all diagonal elements are in 
$]0,1[$;  we will refer to this basis as a 'feature basis'. We state now the main convergence results, one for each protocol.
To this end we need to choose a regime for the initialization of the  diffusion training. We will therefore suppose that the initialization is  non-biased, has finite variance and all generations are independently initialized i.e.,
\begin{align}
Z_n \sim \nu \text{ i.i.d., } \ \ \forall n :  
\mathbb{E}[Z_n]=0, \ Var[Z_n] < \infty.
\label{eq:initialization_hyp}
\end{align}

\begin{theorem}
Under the last generation protocol
 \eqref{eq:initialization_hyp},
 \eqref{eq:training_equation}, \eqref{eq:last_generation}
both the target parameter $T_n$ and the model parameters $S_n$  converge to some 
$T_\infty$ with 
(stationary) distribution $\pi_{T_\infty}^{LG}$ 
with zero average but non-null limit variance.
Moreover, for a feature 
$e_\lambda$ (eigenvector of $M$)
corresponding 
to eigenvalue $\lambda \in ]0,1[$ the 
halving time of the average $m_n = \mathbb{E}[\langle T_n, e_\lambda \rangle]$
is
\begin{align}
h_{1/2}^{last \ gen} =
\left\lceil \frac{\log 2}{- \log|1-\lambda|} \right\rceil.
\label{eq:halving_last_gen}
\end{align} 
Here, $\lceil x \rceil$ denotes the smallest integer greater than or equal to $x$ and halving time is the minimal integer $n$ such that $m_{n} \le \frac{1}{2} m_0$.
\label{thm:last_gen_cv_halving}
\end{theorem}
\begin{proof}
See appendix \ref{appendix:proof_main_results}.
\end{proof}

\begin{theorem}
Under the fixed budget protocol
 \eqref{eq:initialization_hyp},
 \eqref{eq:training_equation}, \eqref{eq:fixed_budget}
the target parameter $T_n$ converges to zero pointwise. Moreover, for a feature 
$e_\lambda$ (eigenvector of $M$)
corresponding 
to eigenvalue $\lambda \in ]0,1[$ the 
halving time of the average $m_n = \mathbb{E}[\langle T_n, e_\lambda \rangle]$
is
\begin{align}
h_{1/2}^{fixed \ budget}  = e^{-\gamma} 2^{1/\lambda} 
(1+ O(\lambda)),
\label{eq:halving_fixed_budget}
\end{align} 
with $\gamma$ the Euler constant.
\label{thm:fixed_budget_cv_halving}
\end{theorem}
\begin{proof}
See appendix \ref{appendix:proof_main_results}.
\end{proof}
\begin{remark}
Both Theorem 
\ref{thm:fixed_budget_cv_halving} and
\ref{thm:last_gen_cv_halving} still hold true 
with straightforward modification if we allow the possibility for some centered error into the dataset representation, i.e. for fixed budget protocol
$(n+1)T_{n+1} = T_n n + S_{n+1} + \zeta_n$ 
or
$T_{n+1} = S_{n+1} + \zeta_n$  for the last generation protocol.
\end{remark}

\begin{remark}
As a further intuition, note that large eigenvalues in the contraction matrix $M$ correspond to features that are lost fast because along eigenvector (feature) $e_\lambda$ with eigenvalue $\lambda \in ]0,1[$ the initial data decay involves a factor $1-\lambda$; larger $\lambda$ means faster decay and conversely, small $\lambda$ correspond to directions that are more resilient. So $\lambda\sim 0$  correspond to "robust" features while $\lambda\sim 1$ to "fragile" feature, consistent with analysis above.    
\end{remark}

\section{Interpretation of the theoretical results} \label{sec:theory}

\noindent {\bf The last generation protocol : }
In this case there is no convergence of the model parameters over the generations, there is just a limit distribution  $\pi_{T_\infty}^{LG}$ different from the original dataset distribution $\mathcal{D}$ or the initialization $\nu$. The convergence is in distribution and not point-wise which means that model parameters will take some ergodic walk following $\pi_{T_\infty}^{LG}$. The variance is not vanishing to zero but tends to some positive constant. Note that the average goes to zero exponentially fast if $Z_n$ are centered. Let us discuss this more in detail: equation~\eqref{eq:halving_last_gen} gives an estimate of the number of generations required  for some feature in the dataset (more precisely its average) to decay by a factor of $2$. This depends 
on the  eigenvalue associated to the feature.
For instance, when this eigenvalue is close to $1$, say $\lambda=0.99$ the number of generations required is  
$h_{1/2}^{last\ gen} =\frac{\log 2}{- \log|1-0.99|} = 0.15\approx 1$; when the eigenvalue is close to $0$, say $\lambda=0.01$ the number of generations is  
$h_{1/2}^{last\ gen} =\frac{\log 2}{- \log|1-0.01|} = 68.97\approx 69$.
So even for small eigenvalues the memory is erased after a limited number of generations (about 70).

\noindent {\bf The fixed budget case :} Here the situation is quite different as all model parameter trajectories converge to zero, same as the variance ; in particular the average also goes to zero but the memory of the initial data behaves as follows: 
when $\lambda=0.99$ 
we obtain
$h_{1/2}^{fixed\ budget}  \sim e^{-\gamma} 2^{1/0.99} \sim 1.13 \approx 1$~; in sharp contrast with previous case,  
when $\lambda=0.01$ 
we obtain
$h_{1/2}^{fixed\ budget}  \sim e^{-\gamma} 2^{1/0.01} \sim 7.12 \times 10^{29}$.
So for small eigenvalues the number of generations required to erase the memory of the initial feature in the dataset is extremely high, in practice  never lost. The dataset characteristics are therefore partitioned into a set of {\bf fragile features} ($e_\lambda$ with $\lambda\sim 1$) that are lost early and {\bf robust features} 
($e_\lambda$ with $\lambda\sim 0$)
that remain for long periods (some of them practically forever).

To summarize, while both protocols rapidly forget high-eigenvalue (fragile) features, they differ sharply for low-eigenvalue features: the last-generation protocol erases memory of a feature within a finite number of generations, whereas the fixed-budget protocol preserves some robust features for extremely long times, leading to a practical persistence of part of the initial dataset information.


\section{Empirical results}
\label{sec:results}

The results can be reproduced using the github available at \url{https://github.com/gabriel-turinici/model_collapse_features}.

We plot here the graphical qualitative results 
depicting the comparison between the initial dataset and the final dataset (scatter plot for 2D and sample images for the others) for both training protocols 
and refer to Appendix~\ref{app:fid-figures} for the quantitative metrics results. 

For the {\bf "last generation" protocol}
the dataset comparison in Figure \ref{fig:dataset_comparison_all}
and metrics evolution in Figure \ref{fig:metrics_last_gen}
show that deterioration is very fast, in agreement with the literature and theoretical results in Section~\ref{sec:theory}.

\begin{figure}[htbp!] \centering \renewcommand{\arraystretch}{1.2} \begin{tabular}{c| c |c |c} 
 & {Reference dataset} & {Last generation protocol} & {Fixed-budget protocol} \\
\hline 
\rotatebox{90}{Spiral} & \includegraphics[width=.12\textwidth,trim={20 20 0 0},clip] {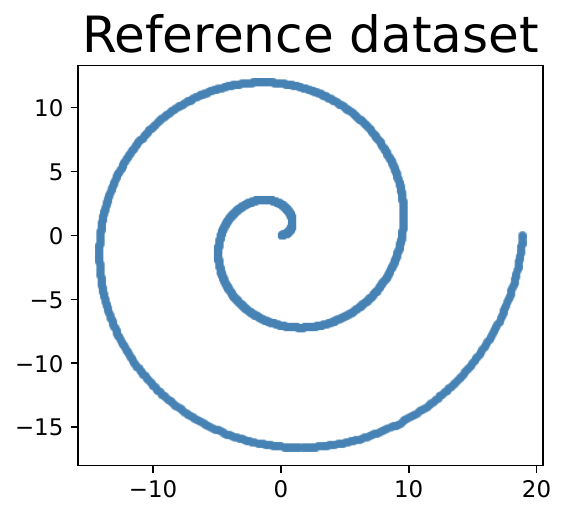} & \includegraphics[width=.10\textwidth,trim={300 20 0 0},clip] {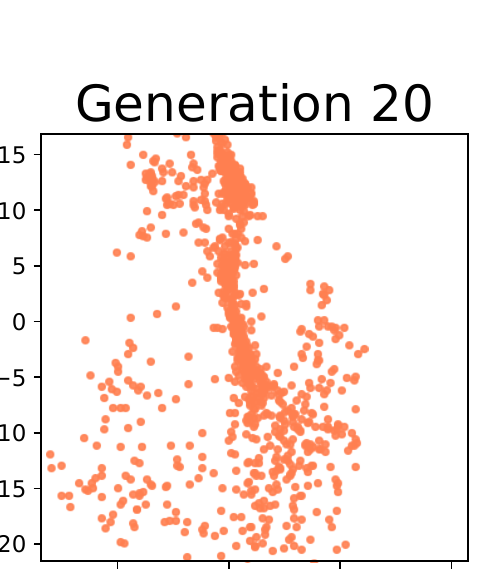} & \includegraphics[width=.10\textwidth,trim={300 20 0 0},clip] {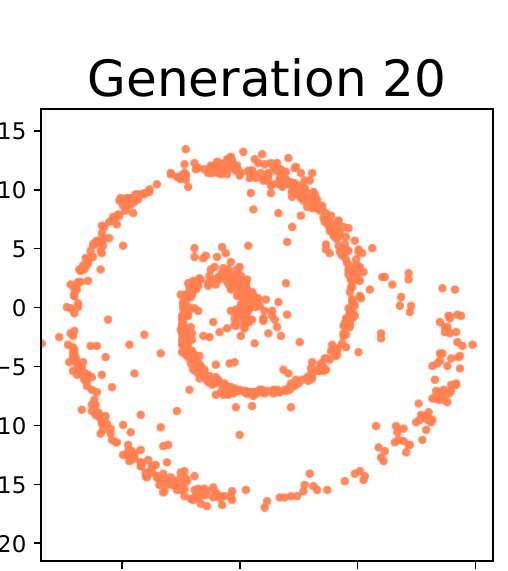} \vspace{-.1cm}\\
%
\rotatebox{90}{MNIST} & \includegraphics[width=.25\textwidth,trim={0 0 1415 0},clip] {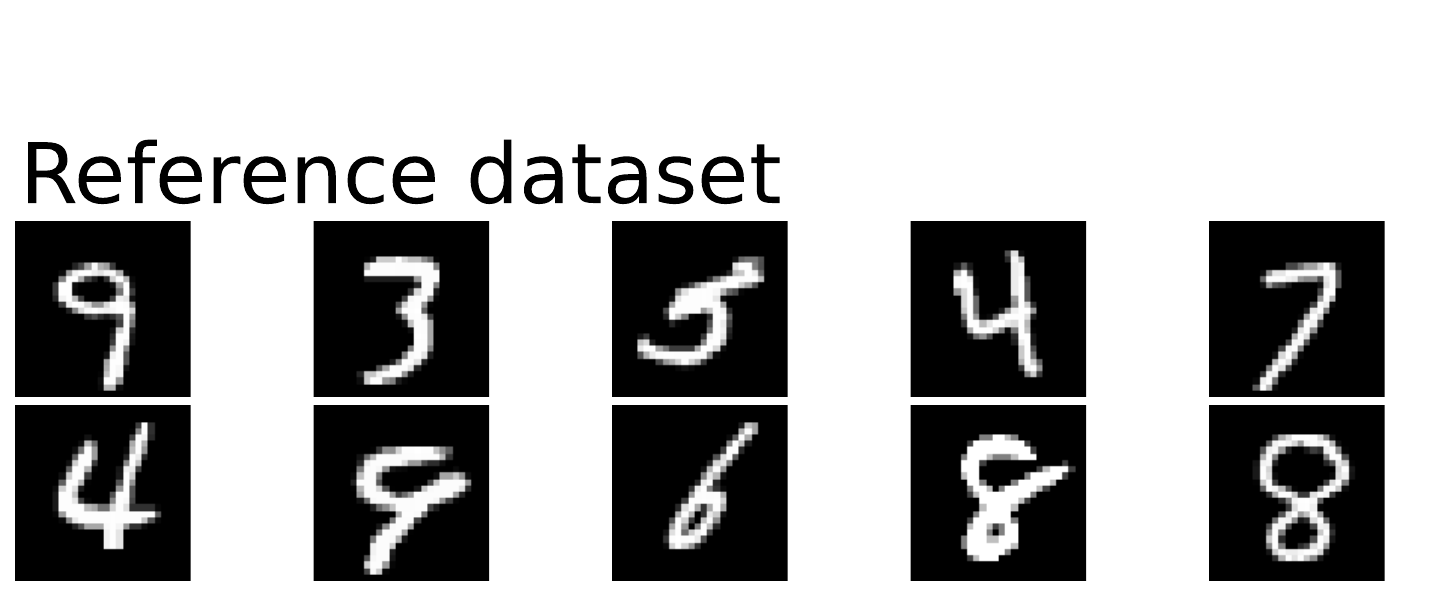} & \includegraphics[width=.25\textwidth,trim={1416 0 0 0},clip] {MNIST_comparison_dataset_last_gen_v2.pdf} & \includegraphics[width=.25\textwidth,trim={1416 0 0 0},clip] {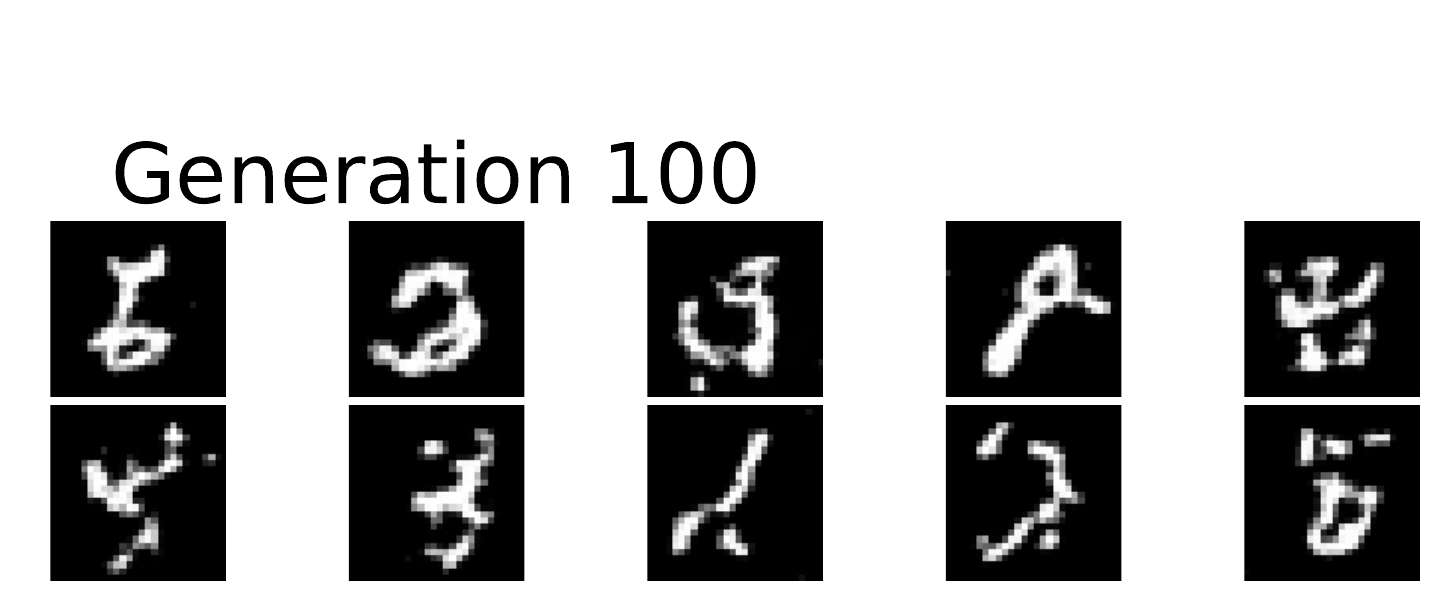} \vspace{-.1cm} \\
\rotatebox{90}{FMNIST} & \includegraphics[width=.25\textwidth,trim={0 0 1415 0},clip] {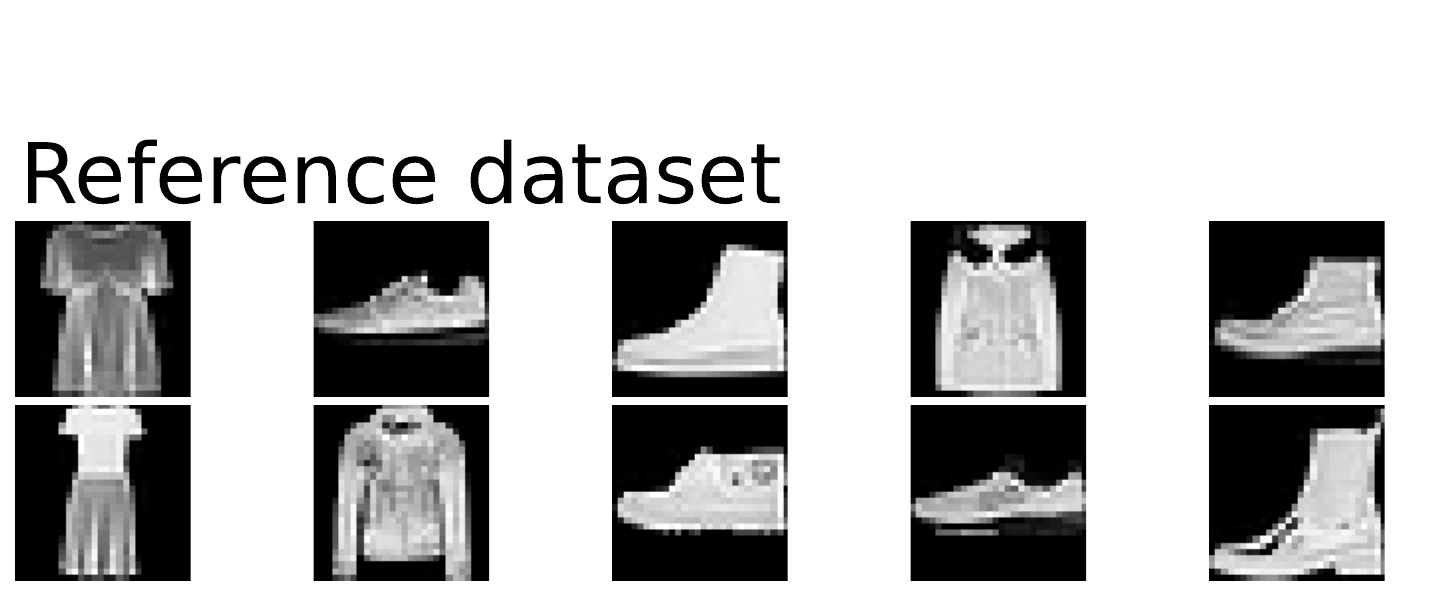} & \includegraphics[width=.25\textwidth,trim={1416 0 0 0},clip] {Fashion-MNIST_comparison_dataset_last_gen_v2.pdf} & \includegraphics[width=.25\textwidth,trim={1416 0 0 0},clip] {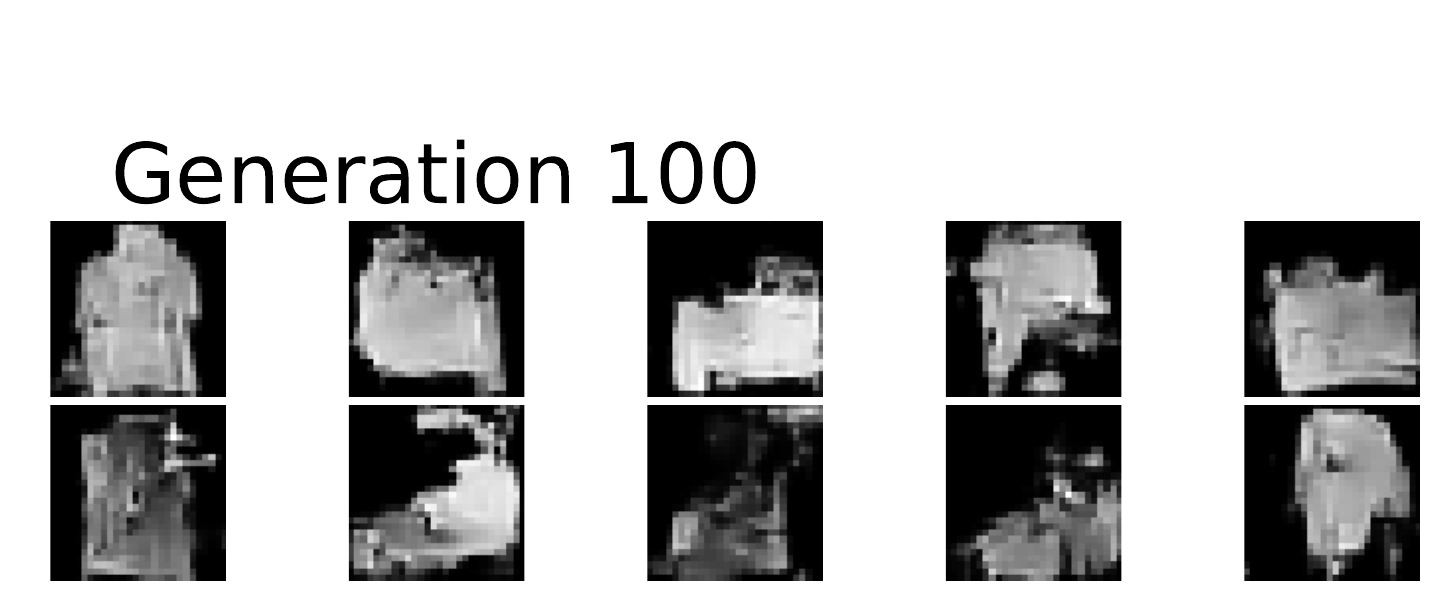} \vspace{-.1cm}\\
\rotatebox{90}{CIFAR-10} & \includegraphics[width=.25\textwidth,trim={0 0 1257 0},clip] {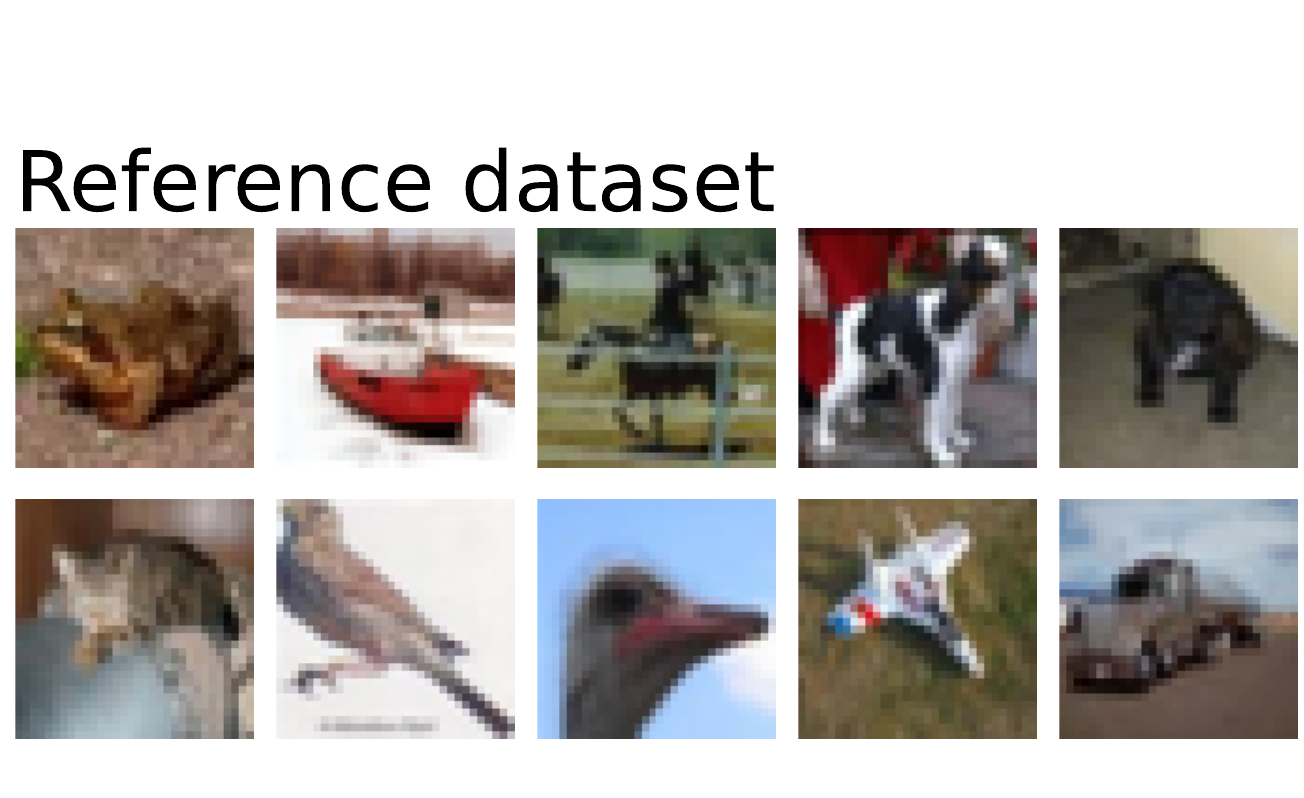} & \includegraphics[width=.25\textwidth,trim={1257.5 0 0 0},clip] {CIFAR10_comparison_dataset_last_gen_v2.pdf} & \includegraphics[width=.25\textwidth,trim={1257.5 0 0 0},clip] {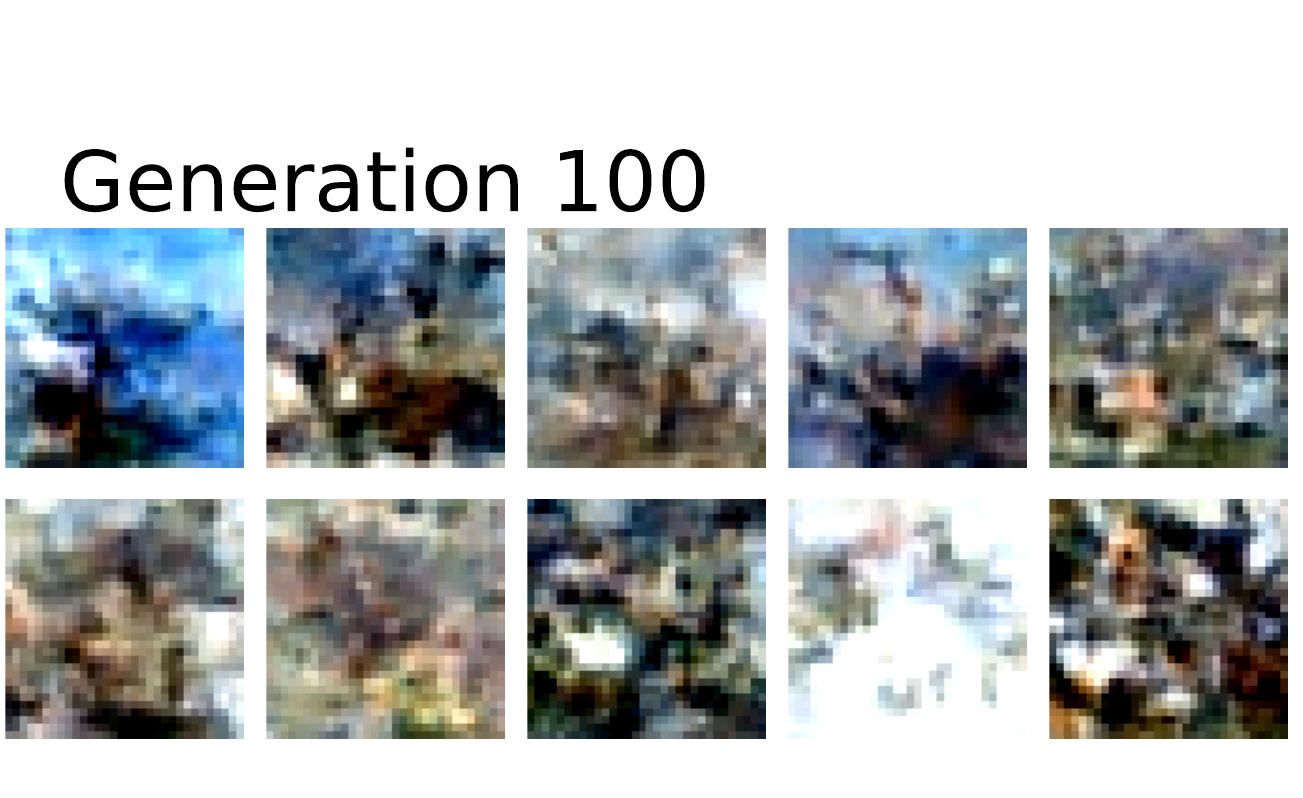} \end{tabular} \caption{ For each dataset we plot the initial data (left column of plots), the samples at the end of the ``last generation'' protocol (center column) and at the end of the ``fixed budget'' protocol (right column). For the ``last generation'' protocol all features are lost leading to an almost white noise random generation, while for the ``fixed budget'' protocol some features are lost but some are kept like the structure of digits (MNIST), clothes (FMNIST) or texture of scenes (CIFAR-10).} \label{fig:dataset_comparison_all} \end{figure}

For the {\bf "fixed budget" protocol}
the dataset comparison in Figure \ref{fig:dataset_comparison_all}
and the metrics evolution in Figure~\ref{fig:metrics_fixed_budget}
show that deterioration is partial and affects only some features as predicted by the theoretical results in Section~\ref{sec:theory}.

The metric evolution is also very discriminative, with last-generation protocol reaching much higher values of FID, that is departing farther from the original dataset, cf. figures in Appendix~\ref{app:fid-figures}.

\section{Discussion and limitations}

\paragraph{What the analysis contributes as principle.}
While the two protocols are usually compared through the \emph{speed} of degradation, 
we place the difference at the per-feature dynamics and obtain theoretical results concerning the feature decay \eqref{eq:halving_last_gen}--\eqref{eq:halving_fixed_budget} which are coherent with the experiments in
Section~\ref{sec:results}. The apparently
divergent conclusions of
\citet{shumailov2024curserecursiontraininggenerated},
\citet{gerstgrasser2024modelcollapseinevitablebreaking},
\citet{seddik2024badtrainingsyntheticdata} and
\citet{barzilai2026modelsdontcollapseconsistency} can now be reconciled for the diffusion model settings.

\noindent {\bf Why these regimes matter in practice.} 
The fragile/robust partition identified theoretically explains which features are at risk. Fragile features are those associated in earlier work with the tails and low-density regions of the distribution (rare and under-represented cases). For anyone training on web-scale corpora that already contain generated content, this turns a diffuse concern into a concrete encouragement to enforce provenance tracking, synthetic-content disclosure, or refreshing the data with real examples. In general 
we hope these insights can help asses and mitigate the effects of model collapse by strengthening preferentially the 'fragile' features. 
On the other hand, the existence of robust features helps explain why collapse in practice is often partial rather than total, potentially even creating a false sense of reassurance that the procedure remains healthy.

\noindent {\bf Limitations.} Our theory is a first-order, NTK-style linear response resting on assumptions that real
recursive training meets only approximately, and the identification of the eigenvectors of $M$
with semantic attributes is an interpretation rather than a consequence of the model.
Empirically, we retrain small models from scratch with unconditional DDIM sampling, without
sweeping the budget size or curating the pool, and our metrics measure distributional distance
only. See Appendix~\ref{app:limitations} for the additional discussion on the limitations.

\section{Conclusion}

We studied recursive diffusion models training in 
two regimes: total replacement ("last generation" protocol) and a "fixed budget" protocol where synthetic data accumulate but the training budget and dataset size stay fixed. This latter regime, introduced as accumulate-subsample by \citet{kazdan2025collapse}, is relevant for finite-budget pipelines. Empirically, on a 2D spiral and on image benchmarks (MNIST, Fashion-MNIST \& CIFAR-10) degradation under a fixed budget is partial and feature-selective, unlike the general rapid collapse of pure replacement. Theoretically, we gave a linear-response framework for this setting, complementary to the stability analysis of \citet{bertrand2024stability}, and show rigorously that the two protocols differ in kind, not degree, and the dataset characteristics split into fragile features, forgotten within a few generations, and robust features that a fixed budget preserves over practically unbounded horizons.  For the principled study of generative modeling, this reconciles previously discordant results; for the trustworthy and sustainable use of these models, it provides a way to anticipate which content, especially the tails and under-represented, fragile,  modes, is lost under recursive training on accumulating data, and how many generations that takes. Identifying features semantically, scaling to large pretrained models and conditional sampling, modeling curation of the pool, and relating halving times to explicit fairness criteria are the natural next steps.


{

\bibliographystyle{plainnat}
\bibliography{refs}
}

\appendix



\section{Related literature and positioning}
\label{sec:related-work}

We recall below some of the works that documented model collapse. The works disagree less about
the recursion itself than about how the training set is built at each step.

\citet{shumailov2024curserecursiontraininggenerated} test on VAE, GMM, and language models and
refit each generation on a finite sample drawn from its predecessor, with no real data retained.
We will call this protocol the "last generation" protocol. They distinguish an early phase, in
which the low-probability regions of the distribution are lost, from a late phase, in which
modes merge and the generated distribution concentrates on a point estimate, and they show
that the finite sample size alone suffices to produce this, independently of any limitation in
the expressivity of the model. We retain the finite per-generation sample, which is the
mechanism they identify, but not the deletion of the real data.

The three works of \citet{dohmatob2024model,dohmatob2024tale,dohmatob2025strong}
analyze model collapse through tractable regression and generative models.
They show that recursive training on synthetic data changes the relevant
scaling laws, primarily because rare events and distributional tails are
progressively lost, leading to bias, un-learning, and eventual collapse.
While the first two works focus mainly on recursive self-training and also
consider mixtures of real and synthetic data, \citet{dohmatob2025strong}
specifically studies a growing data set with a fixed nonzero synthetic
proportion. It shows that performance may fail to improve, or may even
collapse, as the total data size grows, since the synthetic fraction does not
vanish. The first two works are therefore closest to our ``last generation''
protocol, whereas the latter provides a comparison with growing-data
protocols.

\citet{gerstgrasser2024modelcollapseinevitablebreaking} discuss LLMs and generative language
models in general and let the training set grow: generation $n$ is fit on the union of the real
dataset and all previous synthetic generations. The protection they establish comes from this
growth, since a generation that contributes a vanishing share of the training set also
contributes a vanishing share of the noise. It is arguably not a consequence of the presence of the real
data as such. We depart from it on one point: instead of training on the accumulated pool we only train on a fixed size
sample from the generated pool. Note in particular that the fixed size removes the weighting
that makes their argument work.

\citet{kazdan2025collapse} introduce exactly our fixed budget protocol under the name
\emph{accumulate-subsample}, with the same finite-compute motivation, and report across
five generative settings a test loss lying between those of replacement and accumulation
and typically reaching a plateau. Our experiments are consistent with this picture. What
we add is a per-feature account of it: the plateau reflects robust features, whose halving
time grows as $2^{1/\lambda}$, while fragile features are lost as fast as under replacement.

\citet{seddik2024badtrainingsyntheticdata} focus on LLM (autoregressive language models) and
predict a threshold phenomenon, with a maximal proportion of synthetic data below which collapse
is avoided, and show in the purely synthetic regime that the learned distribution converges to a
point mass supported on an outcome that was likely under the original distribution. Because our
protocol drives the real proportion to zero, their framework is the one that most directly
predicts what we observe, and their tails-first ordering is the discrete counterpart of the
feature-selective degradation reported in Section~\ref{sec:results}.

\citet{barzilai2026modelsdontcollapseconsistency} work with generic generative models under
maximum likelihood estimation in an accumulating setting and avoid collapse even as the real
fraction vanishes, but under two conditions that our setting does not meet: the number of
samples drawn per generation must grow with the number of generations, whereas ours is constant,
and the family of distributions must be smooth enough, failing which they construct consistent
families that collapse arbitrarily fast. DDIM sampling has no maximum likelihood property to
begin with, so their positive result does not transfer, while their negative results show that accumulation alone is not a guarantee.

\citet{bertrand2024stability} study likelihood-based generative models, diffusion models
among them, retrained on a mixture of their own samples and a fixed proportion of the real
data. They linearize the retraining map in parameter space around the optimum and prove
stability when the initial model is accurate enough and the real proportion stays above a
threshold. Our framework rests on the same kind of first-order linearization but differs in
three respects: under the fixed budget the real proportion vanishes as $1/(n+1)$, so their
threshold is eventually crossed; each generation is re-initialized at random, which is the
source of the noise term $MZ_n$; and we derive per-feature halving times rather than a
stability condition, which is what separates fragile from robust features.

Applied to autoregressive transformers LLMs, \citet{pmlr25_model_collapse} study mixtures of
real and synthetic text at a fixed synthetic proportion rather than long recursive chains, and
propose token-level editing of human-produced data as a remedy. Their observation that synthetic
data cover only a narrowed portion of the original distribution is the language-model analogue
of the loss of low-density features we report, obtained without any recursion at all.

We now discuss three works that study diffusion models rather than language models.

\citet{yoon2025chaindiffusion} finetune Stable Diffusion v1.5 with LoRA on its own outputs,
restarting each iteration from the same pretrained weights on a fixed prompt set with no real
data retained: a last generation protocol with a deterministic and strongly informative
initialization. Their central finding is that the classifier-free guidance scale governs both
the speed and the direction of degradation, low guidance producing blurring and high guidance
high-frequency artifacts. This mechanism has no counterpart here, where sampling is
unconditional, which is one reason our degradation is feature-selective rather than
directional. Their quantitative-trait analysis yields a drifting mean with variance stabilizing
at a positive value; Theorem~\ref{thm:last_gen_cv_halving} reaches a comparable non-degenerate
limit through an unrelated mechanism.

\citet{bohacek2025nepotistically} retrain the denoising network of Stable Diffusion on faces of
its own creation over five iterations, at fixed synthetic proportions from $3.3\%$ to $100\%$.
Collapse occurs at every proportion, and curation of the retraining set does not slow it down.
Their real proportion stays constant across iterations rather than being driven to zero as in
our fixed budget protocol, so their result bounds how little synthetic data suffices but leaves
the long-horizon asymptotics open. Their healing experiment, where FID recovers while artifacts
persist, indicates that part of the information survives the chain, consistent with the robust
features of Section~\ref{sec:theory}. Both works finetune a large pretrained model over few
iterations, whereas we retrain small models from scratch over one hundred generations, so the
pretrained weights play in their setting the role that the random initialization $Z_n$ plays in ours.%

\citet{shi2025closer} retrain DDPMs from scratch on CIFAR-10, FFHQ and MNIST under both
replacement and accumulate-subsample, over fewer than ten iterations. They identify a
transition from generalization to memorization driven by the decreasing entropy of the
training set, and propose an entropy-based selection of the training subset as a remedy.
Their study is empirical and measures sample novelty; ours follows the same protocol over
one hundred generations and gives a theoretical account of which features survive.
\\ 

\vspace{1cm}

\noindent {\bf Recall on positioning.}
Our contributions 
are empirical and theoretical. Empirically, 
 most model-collapse works study VAEs, GMMs or language models, while we work with diffusion models (DDPM training, DDIM sampling), which have received far less attention. Theoretically, we introduce a new vector-valued mathematical description of the multi-generational parameter dynamics and prove convergence results explaining both behaviors. In particular we recover the rapid collapse of total replacement but find that a fixed training budget causes only partial, feature-selective degradation. 

\section{Classical stochastic recursion convergence results}
We recall here some classical stochastic analysis results that will be used to prove the main results.
\begin{lemma}
Consider the recursion
\begin{align}
Y_{n+1}
=
\left(I-\frac{1}{n}M\right)Y_n
+\frac{1}{n}\epsilon_n,
\end{align}
where $M$ is symmetric positive definite with eigenvalues
$\lambda_i(M) \in ]0,1[$, and $(\epsilon_n) \sim \mu$ are i.i.d. and $Var[\mu]< \infty$; denote 
$\bar \mu = \mathbb{E}[\mu]$. One has
\begin{align}
Y_n \xrightarrow[n\to\infty]{a.s.} M^{-1}\bar\mu.
\end{align}
In particular, if $\mathbb E[\epsilon_1]=0$, then
\begin{align}
Y_n \xrightarrow[n\to\infty]{a.s.} 0.
\end{align}
\label{lemma:stochastic_reccursion}
\end{lemma}
\begin{proof}

We use \cite[Chapters 2 and 4]{Borkar2022} more precisely  convergence theorem
Theorem 2.1 page 17 (we will check conditions (A1)-(A3) page 11) and boundedness Theorem 4.1 page 54, (we check condition (A9) page 50).
Here $h(x)=Mx$. 
Apply the convergence theorem  to
\begin{align}
Y_{n+1}
=
Y_n+\frac{1}{n}
\bigl[-(MY_n-\bar\mu)+(\epsilon_n-\bar\mu)\bigr].
\end{align}

The associated ODE is $\dot{y}=-My+\bar\mu$, whose unique globally asymptotically stable equilibrium is
$y^\ast=M^{-1}\bar\mu$.
 Since $\sum_{n=1}^{\infty} n^{-1}=\infty$, 
$\sum_{n=1}^{\infty} n^{-2}<\infty$, 
and $\epsilon_n-\bar\mu$ has finite second moment, the theorem yields
\begin{align}
Y_n\longrightarrow M^{-1}\bar\mu
\qquad\text{almost surely}.
\end{align}    
Alternative proofs can follow the arguments in  \cite{Duflo1997,KushnerYin2003}.
\end{proof}
\begin{lemma}[Stable linear stochastic recursion]
Let $(X_n)_{n\ge 0}$ be defined by
\begin{align}
X_{n+1}=AX_n+U_n,
\label{eq:reccursion_axplusu}
\end{align}
where $X_n,U_n\in\mathbb R^d$, $(U_n)_{n\ge0}$ are i.i.d. random vectors with
\begin{align}
\mathbb E\|U_0\|^2<\infty,
\end{align}
and $A\in\mathbb R^{d\times d}$ a symetric positive definite matrix with spectrum in $]0,1[$. 
Then there exists a unique stationary distribution $\pi$, and for every initial condition $X_0$,
\begin{align}
X_n \xrightarrow[n\to\infty]{d} X_\infty \sim \pi
\ \text{ with }
X_\infty
\stackrel{d}{:=} 
\sum_{k=0}^{\infty}A^kU_k,
\end{align}
where the series converges in $L^2$, and $\pi$ is the unique 
distribution such that:
\begin{align}
X_\infty \stackrel d= AX_\infty + U,
\end{align}
with $U$ distributed as $U_0$ and independent of $X_\infty$.
In addition: 
\begin{align}
\mathbb E[X_\infty]
=
(I-A)^{-1} \mathbb E[U_0].
\end{align}

If $\Sigma_U=\operatorname{Cov}(U_0)$, then the covariance matrix
$\Sigma_\infty=\operatorname{Cov}(X_\infty)$ is the unique solution of the
discrete Lyapunov equation
\begin{align}
\Sigma_\infty
=
A\Sigma_\infty A^\top+\Sigma_U.
\end{align}
\label{lemma:var1_convergence}
\end{lemma}

\begin{proof}
This is the standard stationary solution theorem for a stable VAR(1)
process; see, e.g., Lütkepohl \cite[Sections 2.1.1--2.1.4]{Lutkepohl2005}. The contraction property of $A$ is one of the main ingredient for the existence and uniqueness of the distribution $\pi$.
Note that only the convergence in distribution hold, not pointwise; moreover the formula 
$X_\infty
\stackrel{d}{:=} 
\sum_{k=0}^{\infty}A^kU_k$ has only a meaning in distribution sense and not pointwise, in particular the order of $U_k$ has no importance (can be permuted) as they do not correspond necessary to the increments in recursion \eqref{eq:reccursion_axplusu}.
\end{proof}

\section{Proof of the main results}
\label{appendix:proof_main_results}
\subsection{Proof of Theorem \ref{thm:last_gen_cv_halving}}
\begin{proof}
We prove convergence for $T_n$, the convergence of $S_n$ being immediate from 
\eqref{eq:last_generation}. 
Under the last generation protocol
after replacing \eqref{eq:last_generation}
in  \eqref{eq:training_equation} we obtain:
\begin{align}
T_{n+1}= A T_n + U_n, \ \text{ with }
A:= I-M, \ U_n := MZ_n + \varepsilon_n.
\end{align}
We can use now Lemma~\ref{lemma:var1_convergence} to obtain the convergence result. For the second part, let us now inquire about the halving time 
$h_{1/2}^{last \ gen}$
of the average $m_n = \mathbb{E}[\langle T_n, e_\lambda \rangle]$. 
Then, since $U_n$ is centered, the recursion for the average $m_n$ of this feature can be written 
$m_{n+1} = (1-\lambda) m_n$ . This means $m_{n} = (1-\lambda)^n m_0$. The time required to reach half of the value $m_0$ will be the smallest $n$ such that $(1-\lambda)^n \le 1/2$ which gives equation \eqref{eq:halving_last_gen}.
\end{proof}
\subsection{Proof of Theorem \ref{thm:fixed_budget_cv_halving}}
\begin{proof}
We consider the fixed budget protocol
 \eqref{eq:initialization_hyp},
 \eqref{eq:training_equation}, \eqref{eq:fixed_budget}.
After replacing \eqref{eq:fixed_budget}
in  \eqref{eq:training_equation} we obtain:
\begin{align}
T_{n+1}=  T_n - \frac{M}{n+1}T_n  + U_n, \ \text{ with }
U_n := \frac{MZ_n + \varepsilon_n}{n+1}.
\end{align}
We can use now Lemma~\ref{lemma:stochastic_reccursion} to obtain that the target parameter $T_n$ converges to zero pointwise. 

For the second part, we study the halving time 
$h_{1/2}^{fixed \ budget}$
of the average $m_n = \mathbb{E}[\langle T_n, e_\lambda \rangle]$. 
Since $U_n$ is centered, the recursion for  $m_n$ can be written
$m_{n+1}=\left(1-\frac{\lambda}{n+1} \right) m_n$
which yields
$m_n/m_0= \prod_{k=1}^{n}  \left(1-\frac{\lambda}{k} \right) = \frac{\Gamma(n+1-\lambda)}{\Gamma(n+1) \Gamma(1-\lambda)} \sim \frac{n^{-\lambda}}{\Gamma(1-\lambda)}$ which show that halving time if $e^{-\gamma} 2^{1/\lambda} 
(1+ O(\lambda))$ where $\Gamma$ is the Euler function.
\end{proof}

\section{Metrics results}
\label{app:fid-figures}

\begin{figure}[H]
  \centering
\includegraphics[height=.2\textheight, trim={0 0 240 0}, clip]{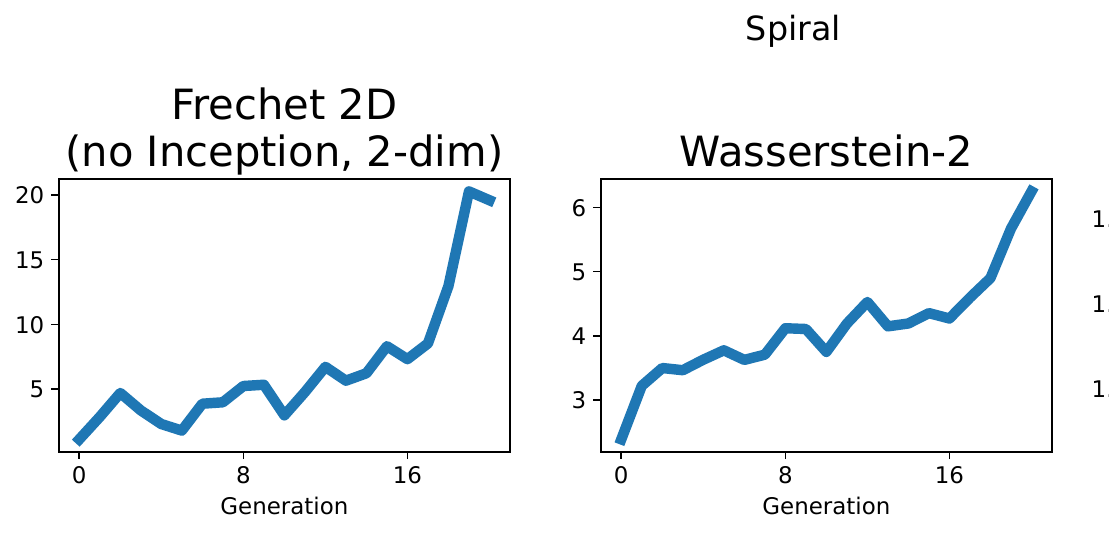}

\includegraphics[height=.2\textheight]{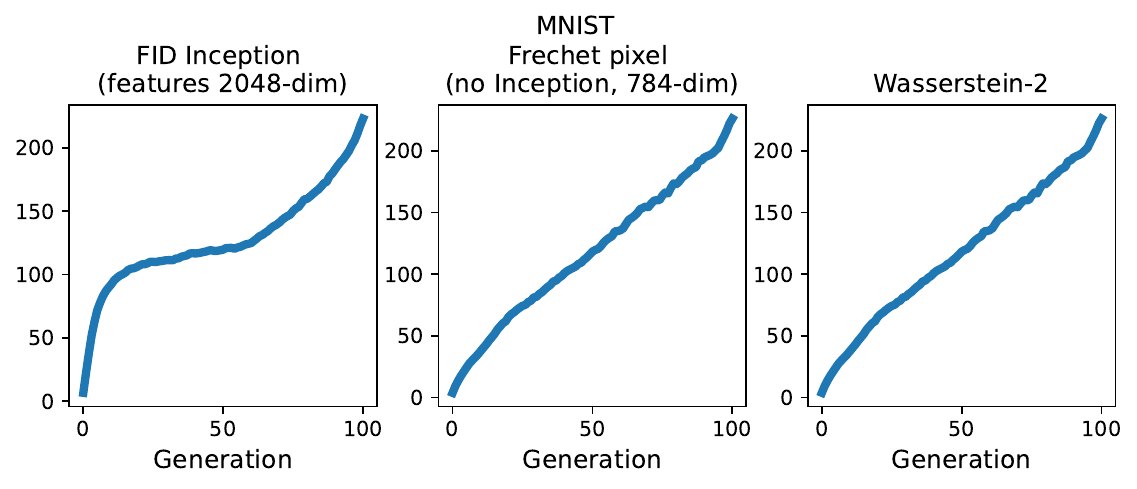}

\includegraphics[height=.2\textheight]{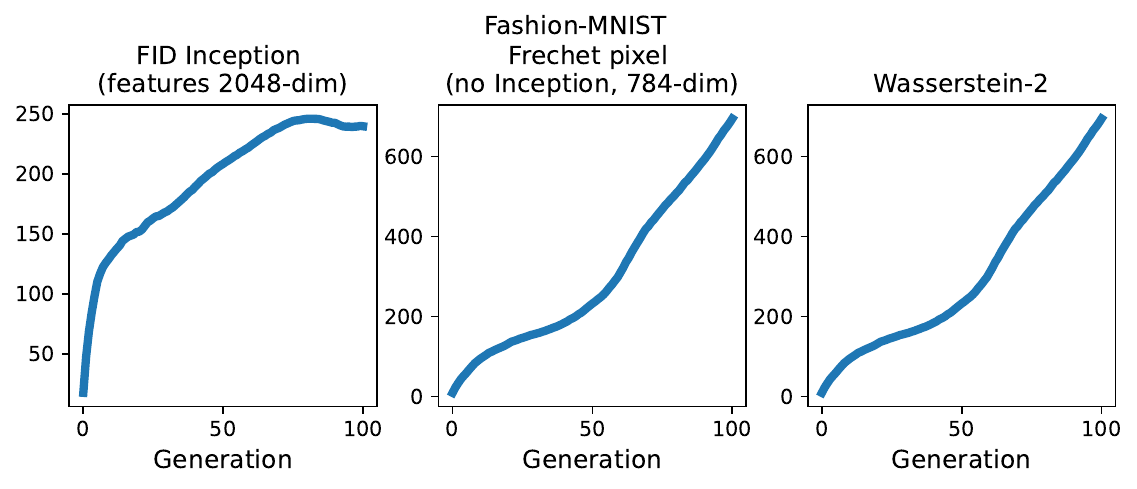}

\includegraphics[height=.25\textheight]{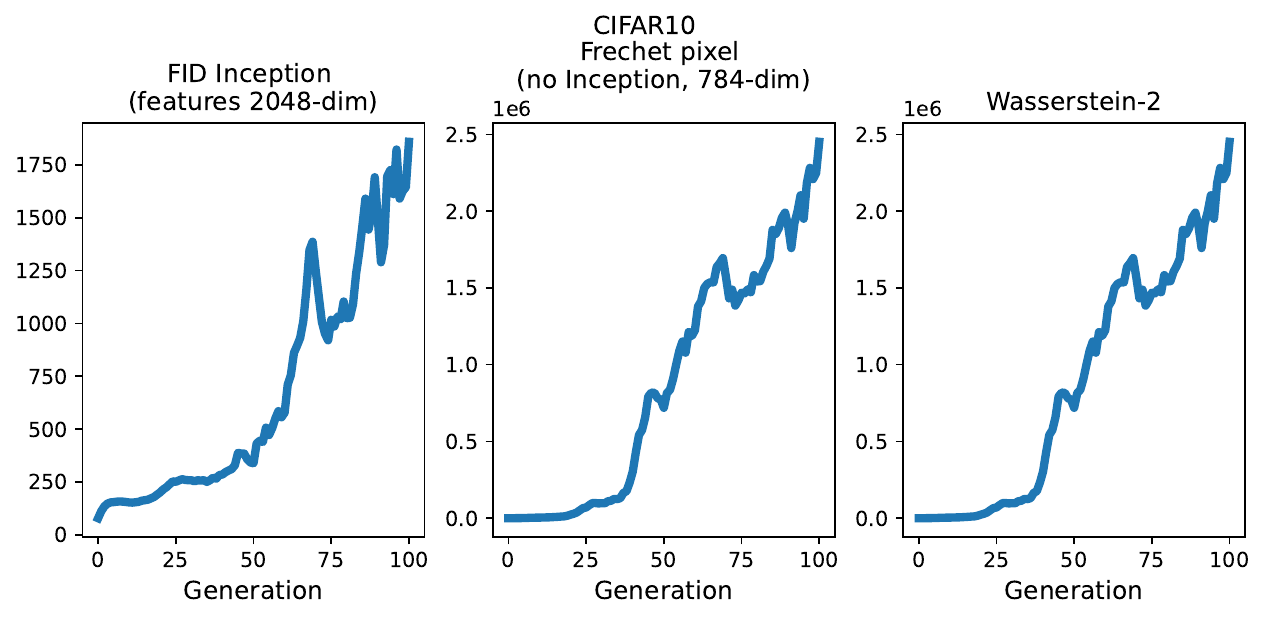}
  \caption{Last generation training protocol cf. Section \ref{sec:methodology}. We plot the evolution of the FID, pixel-FID and Wasserstein-2 metrics. Distance rapidly increases in all cases.
  {\bf From top to bottom:} 2D spiral, MNIST, Fashion-MNIST, CIFAR-10.
  }
  \label{fig:metrics_last_gen}
\end{figure}

\begin{figure}[H]
  \centering
\includegraphics[height=.2\textheight, trim={0 0 240 0}, clip]{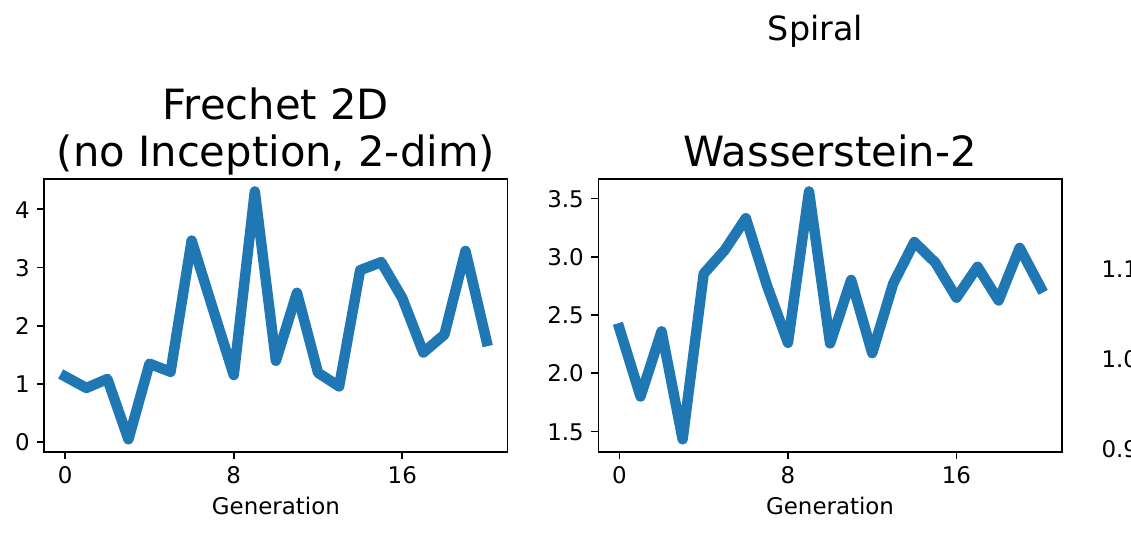}

\includegraphics[height=.2\textheight]{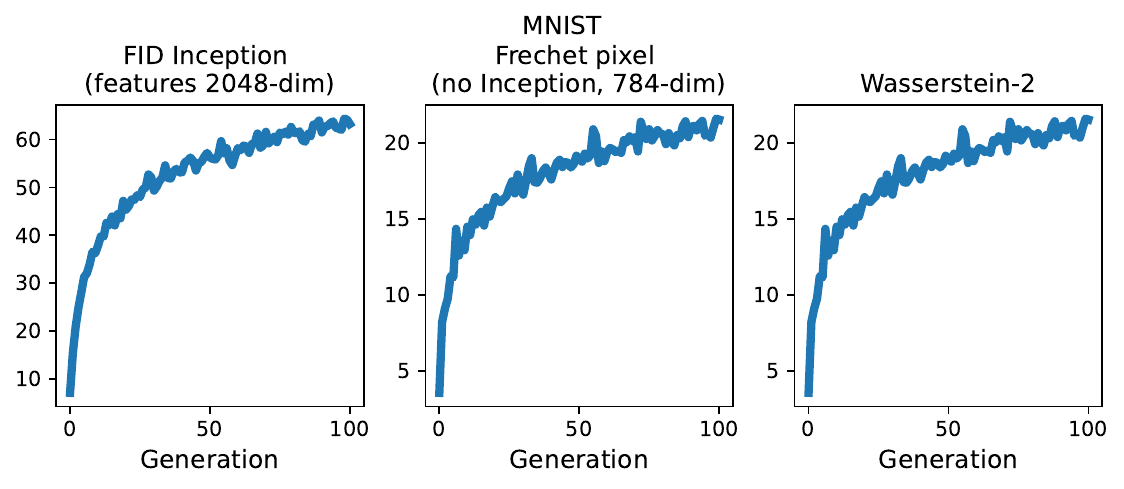}

\includegraphics[height=.2\textheight]{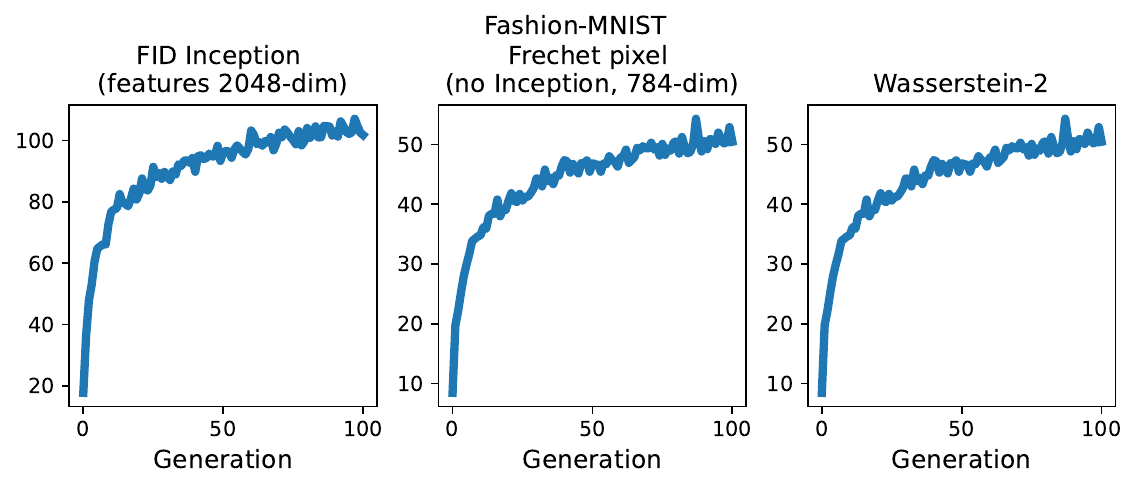}

\includegraphics[height=.25\textheight]{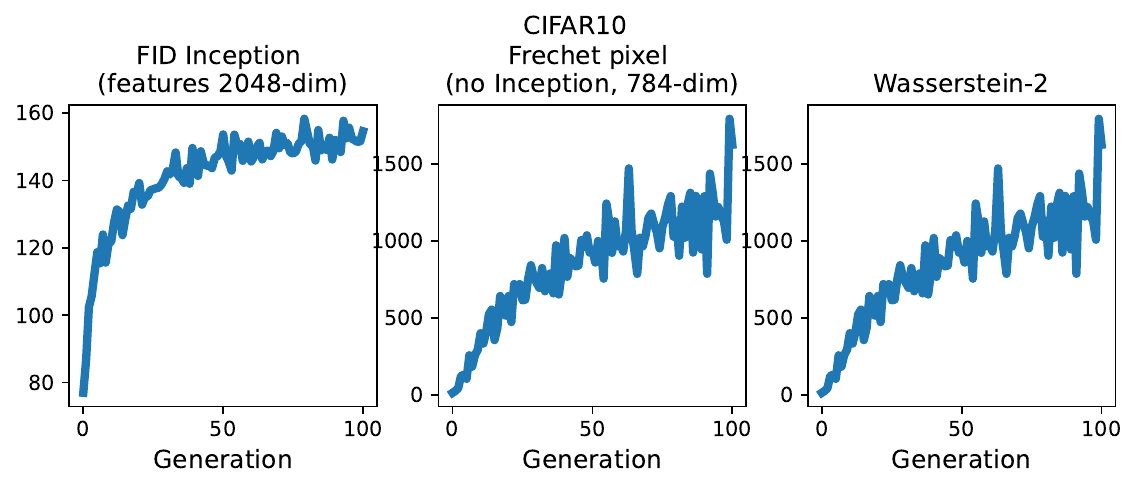}
  \caption{Fixed-budget generation training protocol cf. Section \ref{sec:methodology}. We plot the evolution of the FID, pixel-FID and Wasserstein-2 metrics. Distance increases in all cases but the increase is more tempered than for the last generation protocol in Figure~\ref{fig:metrics_last_gen}. 
  {\bf From top to bottom:} 2D spiral, MNIST, Fashion-MNIST, CIFAR-10.
  }
  \label{fig:metrics_fixed_budget}
\end{figure}

\section{Detailed limitations}
\label{app:limitations}

\begin{itemize}
\item \emph{Modelling assumptions.} The theory is a first-order, NTK-style linear
response. It assumes an exact parameter representative for every dataset
\eqref{hyp:dataset_representation}, linear interpolation in parameter space for
data mixtures, a symmetric positive-definite contraction $M$ with a fixed
spectrum, and centered i.i.d.\ noise and initialization independent across
generations. Real recursive training meets none of these exactly; the remark
after Theorem~\ref{thm:fixed_budget_cv_halving} only absorbs additive
representation error. So it should be rather be read as the simplest linear model that can be relevant to this setting, not the most precise one.
\item \emph{Feature identification.} The features are eigenvectors of $M$. Their
identification with semantic attributes of the data, or with distribution tails,
is an interpretation supported by the experiments and by analogy with
\citet{seddik2024badtrainingsyntheticdata}, not something the model establishes.
\item \emph{Sampler and setting.} We use unconditional DDIM sampling, which has
no maximum-likelihood property, so the positive result of
\citet{barzilai2026modelsdontcollapseconsistency} does not apply and ours does
not transfer to guided or conditional generation, where
\citet{yoon2025chaindiffusion} show the guidance scale drives degradation.
\item \emph{Scale.} We retrain small models from scratch, with $n_d$
samples on several datasets. The behavior
of the fragile/robust partition for large pretrained models finetuned over few
iterations \citep{bohacek2025nepotistically,yoon2025chaindiffusion} is not
tested here.
\item \emph{Pipeline realism.} We do not sweep the budget size or the
pool-mixing ratio, and we do no model curation or filtering of the pool, which
\citet{bohacek2025nepotistically} report can change the picture.
\item \emph{Metrics.} FID, pixel-FID and $\mathcal{W}_2$ measure distributional
distance, not downstream utility or any explicit fairness criterion; connecting
halving times to the latter is left open.
\end{itemize}

\section{Common architecture and code flow}

\subsection{Recursive training protocol.}

Generation 0 is trained from scratch on $n_d=10\,000$ real images (see Section~\ref{sec:appendix_spiral} for 2D-spiral dataset).
After training, the EMA model generates another $n_d$ synthetic samples.

For generation $g>0$, a pool is constructed from

\begin{align}
\mathcal{P}ool_n
=
\mathcal{D}
\cup
\mathcal{D}_1
\cup
\dots
\cup
\mathcal{D}_{n-1},
\end{align}

where $\mathcal{D}_n$ denotes the synthetic dataset produced at generation
$n$.

For the fixed budget protocol
a random subset of $n_d$ images is sampled from this pool
and used to
train a newly initialized diffusion model. 
For the last generation protocol $\mathcal{D}_{n-1}$ is used to train generation $n$.
The trained model then generates
a new synthetic dataset $\mathcal{D}_n$, which is added to the pool for the
next generation.

This process is repeated for 100 generations (see Section~\ref{sec:appendix_spiral} for 2D-spiral dataset). Sample quality is monitored
through the Fréchet Inception Distance (FID) computed against the original
real dataset.

\subsection{DDPM training, DDIM sampling}
The sampling principles are common to all datasets. 
Generation uses deterministic DDIM sampling. Starting from

\begin{align}
x_T \sim \mathcal{N}(0,I),
\end{align}

the model predicts

\begin{align}
\hat{x}_0
=
\frac{x_t-\sqrt{1-\bar{\alpha}_t}\,
\epsilon_\theta(x_t,t)}
{\sqrt{\bar{\alpha}_t}},
\end{align}

and performs the deterministic DDIM update

\begin{align}
x_{t'}
=
\sqrt{\bar{\alpha}_{t'}}\,\hat{x}_0
+
\sqrt{1-\bar{\alpha}_{t'}}
\,\epsilon_\theta(x_t,t).
\end{align}

Although the diffusion process is trained with $T=1000$ steps, inference
uses  $100$ or $200$ DDIM steps (depending on dataset).

\section{Specific architecture and code flow: 2D spiral} \label{sec:appendix_spiral}

\subsection{Dataset}

The real dataset is a two-dimensional Archimedean spiral. For $n_d$ points we take
$\theta$ uniformly spaced on $[0, 2\pi \tau]$ and set

\begin{align}
r = r_0 + s\,\theta,
\qquad
x = r\cos\theta,
\qquad
y = r\sin\theta,
\end{align}

with $r_0 = 0$, spacing $s = 1.5$ and $\tau = 2$ turns. The sampling is deterministic
given the seed, and no normalisation is applied, so the data are supported on a
one-dimensional curve embedded in $\mathbb{R}^2$ with a total variance of order $10^2$.

\subsection{Network Architecture}

The denoising network operates directly on the two-dimensional coordinates and predicts
a two-dimensional output. Since the data carry no spatial structure, the convolutional
U-Net used for the image datasets is replaced by a fully connected network.

\subsection{Time Conditioning}

Diffusion timesteps are encoded using a sinusoidal positional embedding of dimension
$32$, followed by a two-layer MLP:

\begin{align}
e_t
=
\mathrm{MLP}
\left(
\mathrm{SinusoidalEmbedding}(t)
\right).
\end{align}

The projected time representation has dimension

\begin{align}
d_t = 64.
\end{align}

Unlike the image architectures, time information is not injected through FiLM-style
modulation. The time embedding is concatenated to the input coordinates,

\begin{align}
h_0 = [\,x_t \, ; \, e_t\,] \in \mathbb{R}^{2 + d_t},
\end{align}

and the resulting vector is passed to the backbone.

\subsection{MLP Backbone}

The backbone is a four-layer perceptron with ReLU activations and constant hidden width
$d_h = 256$:

\begin{itemize}
\item Linear: $(2 + d_t) \rightarrow d_h$, ReLU,
\item Linear: $d_h \rightarrow d_h$, ReLU,
\item Linear: $d_h \rightarrow d_h$, ReLU,
\item Linear: $d_h \rightarrow 2$.
\end{itemize}

No residual connection, no normalisation layer and no attention mechanism are used.

\subsection{Summary of Main Architectural Hyperparameters}

\begin{itemize}
\item Input dimension: $2$.
\item Time embedding dimension: $32$.
\item Time MLP dimension: $64$.
\item Hidden width: $256$.
\item Number of hidden layers: $3$.
\item Activation function: ReLU.
\item Output dimension: $2$.
\end{itemize}

\subsection{Parameter Count}

The model contains

\begin{align}
1.56\times 10^5
\end{align}

trainable parameters (exactly $155\,522$), of which $6\,272$ belong to the time-embedding
MLP, a size comparable to that of the MNIST and Fashion-MNIST denoiser.

\subsection{Main hyperparameters}

For reproducibility, the most important parameters in the code are:

\begin{itemize}
\item Dataset: 2D Archimedean spiral, $2$ turns.
\item Training set size: $1000$ points.
\item Number of generations : $20$
\item Batch size: $256$.
\item Diffusion steps: $T=1000$, linear schedule
      $\beta_t \in [10^{-4}, 2\times 10^{-2}]$.
\item DDIM inference steps: $100$.
\item Optimiser: Adam.
\item Learning rate: $10^{-3}$ 
\item EMA decay: $0.8$.
\item Training epochs: $2000$.
\end{itemize}

\subsection{Metrics}

In dimension two neither the Inception features nor the pixel-space Fréchet distance are
meaningful, so the metrics of Section~\ref{sec:methodology} are instantiated as follows:
$\mathcal{W}_2$ is computed exactly by solving the discrete optimal transport problem
between the real and generated point clouds, and the Fréchet distance is computed between
the Gaussian approximations of the two clouds in $\mathbb{R}^2$.

\section{Specific architecture and code flow: MNIST and Fashion-MNIST}

\subsection{Network Architecture}

The denoising network is a compact U-Net operating on MNIST/ Fashion-MNIST images
of size $28\times 28\times 1$. The model predicts a single-channel output
having the same spatial resolution as the input.

\subsection{Time Conditioning}

Diffusion timesteps are encoded using a sinusoidal positional embedding of
dimension $32$, followed by a two-layer MLP:

\begin{align}
e_t
=
\mathrm{MLP}
\left(
\mathrm{SinusoidalEmbedding}(t)
\right).
\end{align}

The projected time representation has dimension

\begin{align}
d_t = 64.
\end{align}

Time information is injected into every residual block through FiLM-style
conditioning. The time embedding is mapped to channel-wise scale and shift
vectors,

\begin{align}
(\gamma_t,\beta_t)
=
W e_t,
\end{align}

which modulate the intermediate feature maps according to

\begin{align}
h
\leftarrow
h \odot (1+\gamma_t)
+
\beta_t.
\end{align}

\subsection{Residual Blocks}

Each residual block contains

\begin{enumerate}
\item a $3\times3$ convolution,
\item ReLU activation,
\item FiLM conditioning from the timestep embedding,
\item a second $3\times3$ convolution,
\item ReLU activation,
\item a residual skip connection.
\end{enumerate}

Whenever the channel dimensions differ between input and output, the shortcut
path uses a $1\times1$ convolution for channel projection.

\subsection{U-Net Backbone}

The architecture follows a two-level encoder-decoder design with skip
connections.

\paragraph{Encoder.}

Starting from a $28\times28\times1$ image, the encoder consists of

\begin{itemize}
\item Residual block:
      $1 \rightarrow 32$ channels,
\item $2\times2$ max pooling,
\item Residual block:
      $32 \rightarrow 64$ channels,
\item $2\times2$ max pooling.
\end{itemize}

The bottleneck representation therefore has spatial resolution
$7\times7$ and $64$ channels.

\paragraph{Bottleneck.}

The latent representation is processed by a single residual block

\begin{align}
64 \rightarrow 64.
\end{align}

No self-attention mechanism is used.

\paragraph{Decoder.}

The decoder mirrors the encoder through nearest-neighbor upsampling and
skip connections:

\begin{itemize}
\item Upsample:
      $7\times7 \rightarrow 14\times14$,
\item Concatenate encoder features,
\item Residual block:
      $(64+64)\rightarrow64$,
\item Upsample:
      $14\times14 \rightarrow 28\times28$,
\item Concatenate encoder features,
\item Residual block:
      $(64+32)\rightarrow32$.
\end{itemize}

The final prediction layer is a $1\times1$ convolution mapping

\begin{align}
32 \rightarrow 1
\end{align}

channel, producing the image-sized output.

\subsection{Summary of Main Architectural Hyperparameters}

\begin{itemize}
\item Input resolution: $28\times28$.
\item Number of input channels: $1$.
\item Time embedding dimension: $32$.
\item Time MLP dimension: $64$.
\item Encoder channels: $32 \rightarrow 64$.
\item Bottleneck channels: $64$.
\item Decoder channels: $64 \rightarrow 32$.
\item Activation function: ReLU.
\item Downsampling: MaxPool$(2)$.
\item Upsampling: nearest-neighbor interpolation.
\item Skip connections: U-Net concatenation skips at both resolutions.
\item Output channels: $1$.
\end{itemize}

\subsection{Parameter Count}

The model contains approximately

\begin{align}
1.87\times 10^5
\end{align}

trainable parameters (about $187$k parameters), making it substantially
smaller than the CIFAR-10 attention U-Net used in the companion experiment.

\subsection{Main hyperparameters}
\begin{itemize}
\item Training set size: $10,000$ points.
\item Number of generations : $100$
\item Batch size: $256$.
\item Diffusion steps: $T=1000$, linear schedule
      $\beta_t \in [10^{-4}, 2\times 10^{-2}]$.
\item DDIM inference steps: $100$.
\item Optimiser: Adam.
\item Learning rate: $10^{-4}$ .
\item EMA decay: $0.8$.
\item Training epochs: $1000$.
\end{itemize}

\section{Specific architecture and code flow: CIFAR-10}

\subsection{Diffusion Backbone}

We employ a denoising diffusion model operating on CIFAR-10 images of size
$32\times 32\times 3$.
The network predicts the additive Gaussian noise
$\epsilon_\theta(x_t,t)$ from a noisy image $x_t$ and diffusion timestep $t$.
Training uses the standard noise-prediction objective

\begin{align}
\mathcal{L}
=
\mathbb{E}_{x_0,\epsilon,t}
\left[
\left\|
\epsilon
-
\epsilon_\theta(x_t,t)
\right\|_2^2
\right],
\end{align}

where

\begin{align}
x_t
=
\sqrt{\bar{\alpha}_t}\,x_0
+
\sqrt{1-\bar{\alpha}_t}\,\epsilon,
\quad
\epsilon \sim \mathcal{N}(0,I).
\end{align}

A linear diffusion schedule is used with $T=1000$ timesteps and

\begin{align}
\beta_t \in [10^{-4}, 2\times 10^{-2}].
\end{align}

\subsection{Time Conditioning}

Diffusion timesteps are encoded using sinusoidal embeddings followed by a
two-layer MLP:

\begin{align}
e_t = \mathrm{MLP}(\mathrm{SinusoidalEmbedding}(t)).
\end{align}

The model uses:

\begin{itemize}
\item sinusoidal embedding dimension: $64$,
\item projected time dimension: $32$.
\end{itemize}

Time information is injected into every residual block through FiLM-style
conditioning. For a hidden activation $h$, the time embedding produces
channel-wise scale and shift vectors

\begin{align}
(\gamma_t,\beta_t)
=
W e_t,
\end{align}

and conditioning is applied as

\begin{align}
h \leftarrow h \odot (1+\gamma_t)+\beta_t.
\end{align}

\subsection{Residual Block}

Each residual block consists of

\begin{enumerate}
\item $3\times3$ convolution,
\item GroupNorm,
\item SiLU activation,
\item FiLM conditioning from the timestep embedding,
\item second $3\times3$ convolution,
\item GroupNorm,
\item SiLU activation,
\item residual skip connection.
\end{enumerate}

If the input and output channel dimensions differ, a $1\times1$ convolution
is used in the residual shortcut.

\subsection{Attention U-Net Architecture}

The denoiser follows a compact U-Net architecture with encoder, bottleneck,
self-attention layer, and symmetric decoder.

For the experiments reported in the code, the base width is

\begin{align}
D = 8.
\end{align}

The channel progression is therefore

\begin{align}
8 \rightarrow 16 \rightarrow 32 .
\end{align}

\paragraph{Encoder.}

\begin{itemize}
\item Input: $32\times32\times3$
\item Two residual blocks:
      $3\rightarrow D \rightarrow D$
\item Downsampling by $2\times2$ max pooling
\item Two residual blocks:
      $D\rightarrow 2D \rightarrow 2D$
\item Downsampling
\item Two residual blocks:
      $2D\rightarrow 4D \rightarrow 4D$
\item Downsampling
\end{itemize}

This produces a bottleneck representation of spatial size
$4\times4$ with $4D$ channels.

\paragraph{Bottleneck.}

The bottleneck contains two residual blocks

\begin{align}
4D \rightarrow 4D \rightarrow 4D.
\end{align}

A multi-head self-attention layer is then applied after flattening the
$4\times4$ spatial grid into $16$ tokens.

\begin{itemize}
\item embedding dimension: $4D=32$,
\item number of heads: $4$.
\end{itemize}

The attention output is added back through a residual connection.

\paragraph{Decoder.}

The decoder mirrors the encoder:

\begin{itemize}
\item Upsample $4\times4 \rightarrow 8\times8$
\item Concatenate encoder skip features
\item Residual blocks:
      $(4D+4D)\rightarrow4D\rightarrow4D$
\item Upsample $8\times8 \rightarrow 16\times16$
\item Concatenate skip features
\item Residual blocks:
      $(4D+2D)\rightarrow2D\rightarrow2D$
\item Upsample $16\times16 \rightarrow 32\times32$
\item Concatenate skip features
\item Residual blocks:
      $(2D+D)\rightarrow D\rightarrow D$
\end{itemize}

Finally, a $1\times1$ convolution maps the hidden representation back to
three channels to predict noise:

\begin{align}
D \rightarrow 3.
\end{align}

\subsection{Main hyperparameters}

\begin{itemize}
\item Training subset size: $10,000$ images.
\item Batch size: $256$.
\item Diffusion steps: $T=1000$.
\item DDIM inference steps: $200$.
\item Learning rate: $5\times 10^{-4}$.
\item EMA decay: $0.8$.
\item Base channel width: $D=8$.
\item Time embedding dimension: $64$.
\item Projected time dimension: $32$.
\item Attention heads: $4$.
\item Training epochs: $1000$.
\end{itemize}

\end{document}